\documentclass{article}

\usepackage{PRIMEarxiv}
\usepackage{natbib}

\makeatletter
\renewcommand*{\@xthm}[2]{%
  \def\@jmlr@currentthm{#1}%
  \@ifundefined{hyper@makecurrent}{}{%
    \hyper@makecurrent{#1}%
    \Hy@raisedlink{\hyper@anchorstart{\@currentHref}\hyper@anchorend}%
  }%
  \@begintheorem{#2}{\csname the#1\endcsname}%
  \ignorespaces
}

\def\@ythm#1#2[#3]{%
  \def\@jmlr@currentthm{#1}%
  \@ifundefined{hyper@makecurrent}{}{%
    \hyper@makecurrent{#1}%
    \Hy@raisedlink{\hyper@anchorstart{\@currentHref}\hyper@anchorend}%
  }%
  \@opargbegintheorem{#2}{\csname the#1\endcsname}{#3}%
  \ignorespaces
}

\makeatother

\usepackage[utf8]{inputenc} 
\usepackage[T1]{fontenc}    
\usepackage{booktabs}       
\usepackage{amsfonts}       
\usepackage{nicefrac}       
\usepackage{microtype}      
\usepackage{xcolor}         

\usepackage{mathtools}
\usepackage{bm,bbm}
\usepackage{float}
\usepackage{balance}
\usepackage{multirow} 
\usepackage{colortbl}
\usepackage{lmodern,tikz}
\definecolor{ink}{HTML}{233348}
\definecolor{grid}{HTML}{CDD3DB}
\definecolor{bumpone}{HTML}{2474AC}
\definecolor{bumptwo}{HTML}{B76220}
\definecolor{bumpthree}{HTML}{328369}
\definecolor{bumpfour}{HTML}{875AAF}
\usetikzlibrary{arrows}
\usepackage{caption}
\usepackage{hyperref}
\usepackage{cleveref}

\usepackage[size=footnotesize]{todonotes}

\makeatletter
\newcommand\Autoref[1]{\@first@ref#1,@}
\def\@throw@dot#1.#2@{#1}
\def\@set@refname#1{
    \edef\@tmp{\getrefbykeydefault{#1}{anchor}{}}%
    \xdef\@tmp{\expandafter\@throw@dot\@tmp.@}%
    \ltx@IfUndefined{\@tmp autorefnameplural}%
         {\def\@refname{\@nameuse{\@tmp autorefname}s}}%
         {\def\@refname{\@nameuse{\@tmp autorefnameplural}}}%
}
\def\@first@ref#1,#2{%
  \ifx#2@\autoref{#1}\let\@nextref\@gobble
  \else%
    \@set@refname{#1}
    \@refname~\ref{#1}
    \let\@nextref\@next@ref
  \fi%
  \@nextref#2%
}
\def\@next@ref#1,#2{%
   \ifx#2@ and~\ref{#1}\let\@nextref\@gobble
   \else, \ref{#1}
   \fi%
   \@nextref#2%
}

\makeatother



\newcommand{\real}{\mathbb{R}}  
\def\P{\mathbb{P}} 
\def\T{{ \mathrm{\scriptscriptstyle T} }} 
\def\ind{\mathbbm{1}} 
\bmdefine\e{e}
\bmdefine\u{u}
\bmdefine\ba{a}
\bmdefine\bbeta{\beta}

\newcommand{\defeq}{:=}

\renewcommand{\hat}{\widehat}  
\renewcommand{\tilde}{\widetilde}

\makeatletter
\DeclarePairedDelimiterX{\norm}[1]{\lVert}{\rVert}{#1} 
\DeclarePairedDelimiterX{\bnorm}[1]{\biggl\lVert}{\biggr\rVert}{#1} 
\DeclarePairedDelimiterX{\abs}[1]{\lvert}{\rvert}{#1} 
\DeclarePairedDelimiterX{\ip}[2]{\langle}{\rangle}{#1,#2} 

\newcommand{\RNum}[1]{\uppercase\expandafter{\romannumeral #1\relax}} 

\newcommand*{\indep}{%
  \mathbin{%
    \mathpalette{\@indep}{}%
  }%
}
\newcommand*{\nindep}{%
  \mathbin{
    \mathpalette{\@indep}{/}%
  }%
}
\newcommand*{\@indep}[2]{%
  \sbox0{$#1\perp\m@th$}
  \sbox2{$#1=$}
  \sbox4{$#1\vcenter{}$}
  \rlap{\copy0}
  \dimen@=\dimexpr\ht2-\ht4-.2pt\relax
  \kern\dimen@
  \ifx\\#2\\%
  \else
    \hbox to \wd2{\hss$#1#2\m@th$\hss}%
    \kern-\wd2 %
  \fi
  \kern\dimen@
  \copy0 
}

\makeatother

\newtheorem{theorem}{Theorem}
\newtheorem{proposition}{Proposition}
\newtheorem{lemma}{Lemma}

\newtheorem{conjecture}{Conjecture}
\newtheorem{example}{Example}
\newtheorem{definition}{Definition}

\newtheorem{remark}{Remark}
\newenvironment{proof}
{\par\noindent\textbf{Proof}\ }
{\hfill$\square$\par}

\newenvironment{proofof}[1]{\begin{proof}\textbf{of {#1}}}{\end{proof}}

\newenvironment{restatement}[2][]
{%
\begin{trivlist}
\item[\hskip\labelsep\bfseries
\autoref{#2}%
\ifx\\#1\\%
.%
\else\ (#1).%
\fi]
}
{\end{trivlist}}

\allowdisplaybreaks

\title{On the Sample Complexity of Active Learning with Membership Queries}

\author{
  Ganghua Wang\thanks{Part of this work was done while the author was a postdoctoral researcher in the Data Science Institute at the University of Chicago.}\\
    Department of Mathematics \\ University of Arizona \\
   \texttt{ganghua@arizona.edu} 
   \And
      Shaddin Dughmi\thanks{Supported by NSF Grant CCF-2432219. Part of this work was done while the author was on sabbatical as the Carter and Tania Neild visiting professor at Northwestern University, as well as a visiting professor in the Data Science Institute at the University of Chicago. } \\
    Department of Computer Science \\ University of Southern California \\
   \texttt{shaddin@usc.edu}
}

\begin{document}

\maketitle

\begin{abstract} 
    This work revisits a fundamental question in active learning: how powerful is the ability to synthesize arbitrary queries? Compared to pool-based active learning, where the learner only selects queries from a given unlabeled pool, it has been observed that this seemingly mild change in query ability may dramatically alter the difficulty of statistical learning. In particular some hypothesis classes, such as two dimensional halfspaces, permit only polynomial error decay in the number of samples when the learner must draw its labels from the pool, yet become exponentially learnable once synthesized queries are allowed. This striking gap suggests that membership query synthesis induces a fundamentally different mode of learning, one that is not adequately captured by existing active learning theory and calls for new analytical tools to characterize its sample complexity for general hypothesis classes. Motivated by this phenomenon, we develop several sufficient conditions, present intriguing examples, and propose a conjectural perspective toward understanding which hypothesis classes admit efficient learning through synthesized queries. 
\end{abstract}


\section{Introduction}\label{sec:intro}
    One of the central goals of machine learning is to reduce the amount of data required to learn accurately. Active learning offers a striking possibility toward this goal: instead of passively receiving training data, the learner can \emph{interactively choose} which data points to query~\citep{angluin1988queries, lewis1995sequential, cohn1996active, settles2009active}. This ability to ask informative questions can dramatically accelerate learning. A classical example is learning a threshold function $f_{\theta}(x) = \ind_{x>\theta}$ on $[0, 1]$. A passive learner requires $O(1/\epsilon)$ labeled samples to estimate $\theta$ up to accuracy $\epsilon$, since informative samples near the threshold appear only rarely. In contrast, an active learner can closely approximate $\theta$ using only $O(\log(1/\epsilon))$ queries through binary search. This exponential improvement illustrates the remarkable potential of active learning.
    A natural and fundamental question is therefore: 
    \begin{center}
        \emph{When can active learning achieve such exponential improvement in sample complexity? }
    \end{center}
    Following the prominent PAC-learning framework~\citep{valiant1984theory}, \emph{sample complexity} quantifies the number of labeled examples required to learn the target hypothesis to a desired accuracy level.
    In passive learning, the sample complexity is by now relatively well understood, characterized by combinatorial quantities such as the VC dimension for many learning problems~\citep{vapnik1971uniform, blumer1989learnability}.


For active learning, however, it has been observed that the answer depends crucially on \emph{how much freedom} the learner has in selecting its queries. For example, in \emph{pool-based} setting~\citep{lewis1995sequential}, the learner may only choose which labels to request from a given unlabeled dataset; while under \emph{membership query synthesis}~\citep{angluin1988queries}, or simply membership queries, the learner may query \emph{arbitrary} points from the input space.     While the distinction between may initially appear mild, it can fundamentally alter the difficulty of learning even for halfspaces in more than one dimension~\citep{hanneke2015minimax, hopkins2020point}. 
 Other models provide different forms of interaction which can dramatically alter the sample complexity of active learning, such as comparison queries~\citep{kane2017active}, equivalence queries~\citep{angluin1988queries}, and region queries~\citep{kontonis2024active}. In this paper, we focus on membership query synthesis, and in particular its contrast to the most traditional setting of pool-based active learning.

    The sample complexity of pool-based active learning has been extensively studied in the literature~\citep{cohn1994improving, dasgupta2004analysis, dasgupta2005analysis, hanneke2015minimax}. Among the most influential results, \citet{hanneke2015minimax} established minimax-optimal rates under various noise conditions and identified the \emph{star number}, or equivalently, the \emph{teaching dimension}, as a fundamental complexity measure. In the realizable (noiseless) setting, the star number completely characterizes when exponential error decay is achievable: a hypothesis class admits exponentially fast learning if and only if its star number is finite.

    Perhaps surprisingly, this picture can break down once membership queries are allowed. In fact, certain hypothesis classes that are provably difficult to learn under pool-based active learning become exponentially learnable with membership queries. 
    We demonstrate one such example below.

    \newcommand{\lin}{\textit{lin}}
    \begin{example}[Linear Classifier (Halfspaces)]\label{ex:lin}
        Let $\mathcal{H}_{\lin}^d$ denote the class of binary linear classifiers supported on $[0, 1]^d$, where $d$ is the dimension. Namely, $\mathcal{H}_{\lin}^d \defeq \{f: f(x) = \ind_{ a^\T x + b\geq 0}, a\in \real^d, b\in \real\}$. For $d=2$, the class $\mathcal{H}_{\lin}^2$ cannot achieve exponential learning rates in the pool-based setting. In contrast, with membership queries, $\mathcal{H}_{\lin}^2$ admits exponentially fast learning rates.
    \end{example}

    To better understand the above result, we also present the following example. 
    
    \newcommand{\intv}{\textit{int}}
    \begin{example}[Intervals]\label{ex:int}
        Let $\mathcal{H}_{\intv} = \{\ind_{x\in[a,b]}: 0\leq a\leq b\leq 1\}$ be the class of interval indicators on $[0, 1]$. Then, $\mathcal{H}_{\intv}$ cannot achieve exponential learning rates in either the pool-based or membership-query setting. 
    \end{example}

\begin{figure}[t]
\centering

\begin{minipage}[c]{0.48\textwidth}
    \centering

    \includegraphics[width=\linewidth]{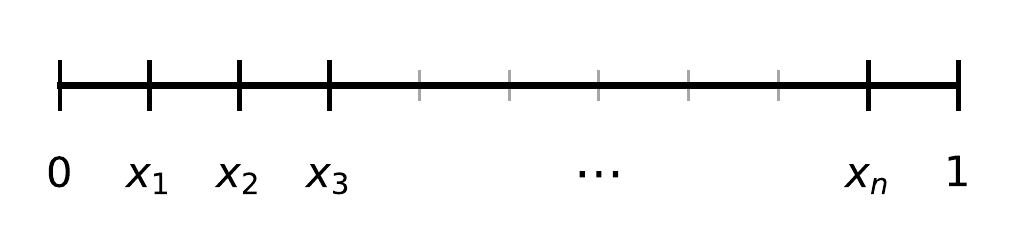}
    {\small (a)}
    \vspace{0.5em}

    \includegraphics[width=\linewidth]{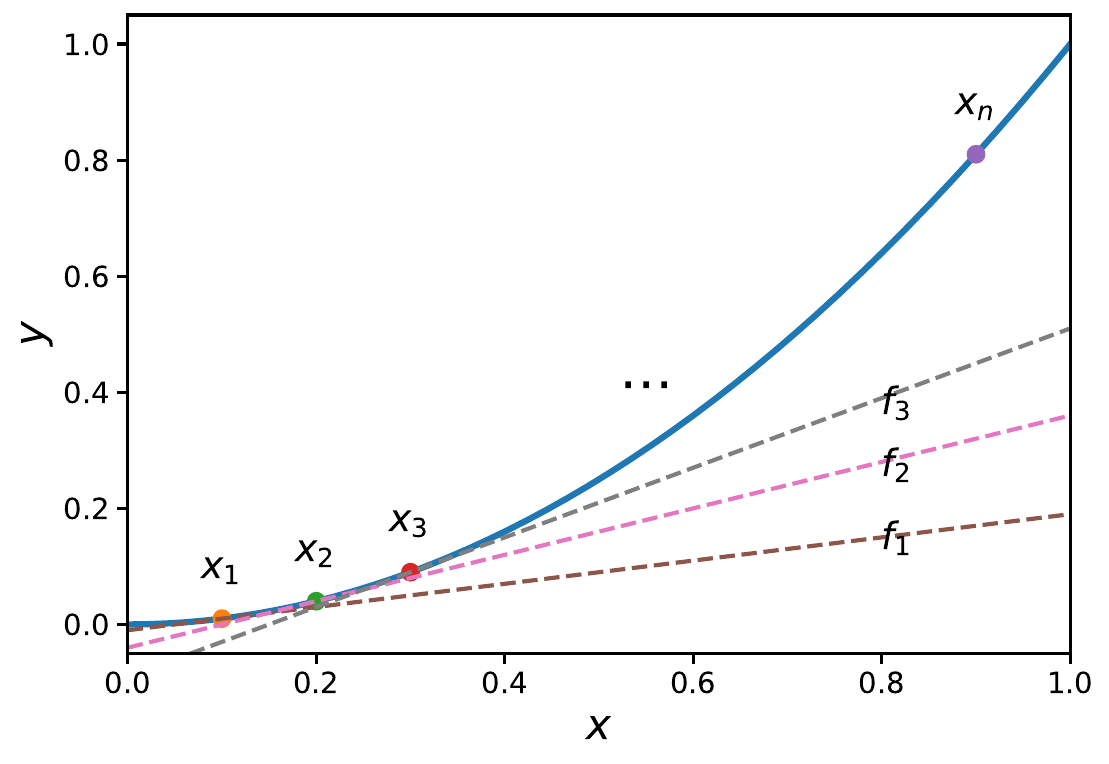}
    {\small (b)}
    
\end{minipage}
\hfill
\begin{minipage}[c]{0.48\textwidth}
    \centering
    \vspace{0.5em}
    \includegraphics[width=\linewidth]{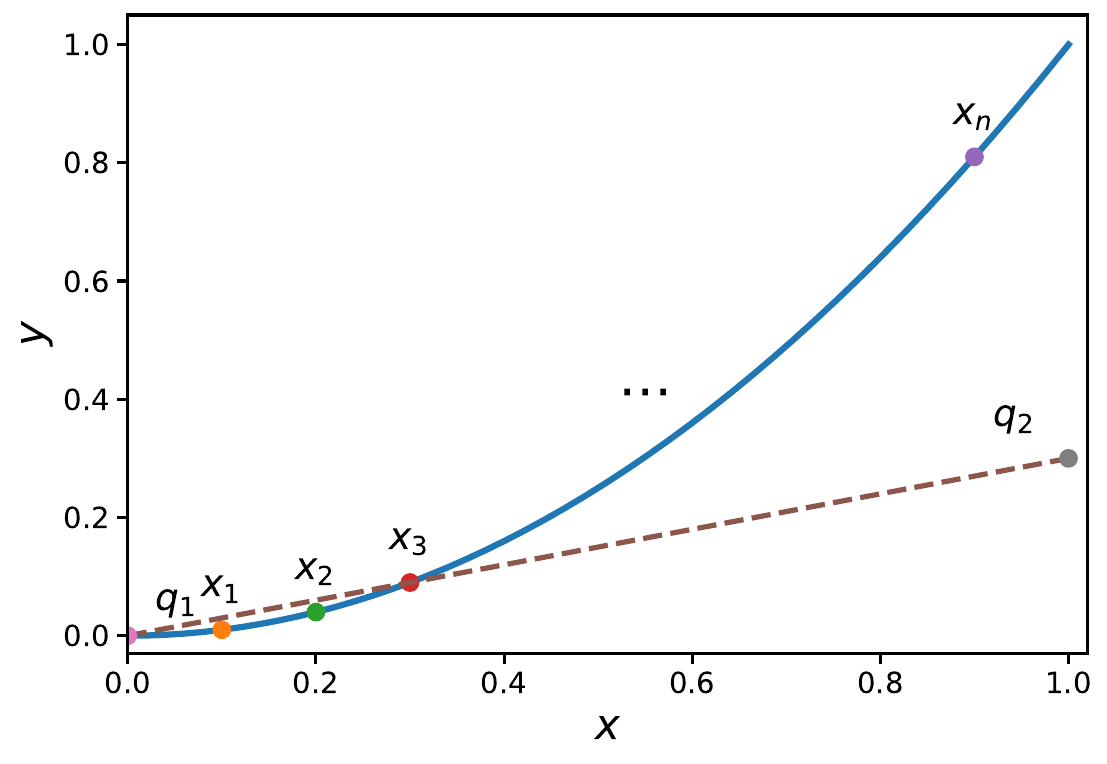}
    {\small (c)}
\end{minipage}

\caption{Embedding intervals into halfspaces. (a) Learning singleton intervals on a discrete set requires identifying the unique positive point, which takes $\Omega(n)$ queries in the worst case. (b) By embedding the points onto a curve, every singleton interval induced by $x_i$ is realizable by a linear classifier $f_i$. Consequently, the same $\Omega(n)$ lower bound applies to pool-based learning of 2D halfspaces. (c) Membership query allows the learner to query arbitrary points outside the given set, such as $q_1$ and $q_2$. It exploits the global geometry and enables a binary-search-like strategy, leading to exponentially faster learning.} 
\label{fig:main_ex}
\end{figure}

    We illustrate the intuition in \Autoref{fig:main_ex}. First, \Autoref{fig:main_ex}(a) demonstrates that the class of intervals is intrinsically difficult to learn, because informative points can be arbitrarily localized. To see this, consider the unlabeled set $S=\{x_i = i/n, i=1,\dots,n\}$. For each $i$, define $f_i(x) = \ind_{x= i/n}$. These functions induce $n$ distinct labelings on $S$, differing only at a single point. Consequently, any learner must identify which location carries the positive label, requiring $\Omega(n)$ queries in the worst case. Equivalently, achieving error of order $1/n$ requires at least $\Omega(n)$ query complexity. 

    The key observation is that this hard interval structure can be embedded into the class of 2D linear classifiers in the pool-based setting, as illustrated in \Autoref{fig:main_ex}(b). By mapping the points onto a convex curve, each singleton labeling can be realized by a linear classifier.
    As a result, when the learner is restricted to querying labels only from the pool $S$, the geometric structure of linear classifiers provides no additional advantage, and the same lower bound carries over. 

    However, membership queries change the situation entirely. As shown in \Autoref{fig:main_ex}(c), by allowing the learner to query arbitrary points, the learner can effectively exploit the linear structure and efficiently localize the decision boundary, leading to exponentially fast convergence rates. 

    This example highlights a fundamental distinction between pool-based active learning and membership query synthesis: allowing the learner to synthesize queries outside the sampled pool can permit exponential learning rates for hypothesis classes which must suffer linear learning rates otherwise. 
    Consequently, existing complexity measures such as the star number are insufficient to characterize exponential learnability under membership query synthesis. Motivated by this gap, we seek new perspectives and complexity measures for understanding the power of membership queries. In particular, we focus on binary classification in the realizable setting, which already exhibits rich and surprising behavior, and leave the study of more general noise models and responses to future work.

    Our main contributions are summarized as follows. 
    \begin{enumerate}
        \item We show that a natural extension of the teaching dimension from the pool-based setting to membership query synthesis does not provide a complete characterization of exponential learnability. 
        \item We then investigate the halving property, a key ingredient behind exponential learning in the pool-based setting, and propose the synthesized halving dimension as a candidate complexity measure for membership query synthesis. For infinite concept classes, we conjecture that a constant synthesized halving dimension is both necessary and sufficient for exponential learning. 
    \end{enumerate}

    Historically, membership query synthesis has often been viewed as less practical and effective because synthesized queries may correspond to unnatural or difficult-to-interpret inputs~\citep{angluin1988queries, baum1992query}. 
    Nonetheless, we believe that understanding what can be learned when arbitrary queries are permitted might help outline the fundamental limits of interactive learning and provide insight into how query power shapes learning difficulty. 

    The rest of paper is organized as follows. We review related work in \Autoref{sec:lit} and formalize the problem setup in \Autoref{sec:prelim}. In \Autoref{sec:td}, we examine a natural generalization of the teaching dimension to membership queries and show its limitations. In \Autoref{sec:half}, we study the halving property and discuss its potential role as a complexity measure. 
    Finally, \Autoref{sec:con} concludes the paper with further discussions.

\section{Related Work}\label{sec:lit} 

There is a rich literature on the sample complexity of active learning in a variety of query models. As discussed in the introduction, the complexity depends critically on the queries available to the learner. We will outline related work in the two models most related to our paper, pool-based active learning and membership query synthesis. We will also briefly mention other enriched query models which seek to circumvent the limitations of pool-based active learning.  

For pool-based active learning, we refer the interested reader to \citet{hanneke2015minimax}, which provides a comprehensive treatment of the subject. In particular, it establishes a complete characterization of sample complexity in the realizable setting and derives some of the best-known bounds under a broad range of noise conditions. The work also relates several influential complexity measures developed in earlier literature, including splitting index~\citep{dasgupta2005analysis}, teaching dimension~\citep{goldman1995complexity, hanneke2007teaching}, and version space compression~\citep{el2010foundations}. 

In contrast, the theory of membership query synthesis is significantly less developed, especially for infinite and nonlinear hypothesis classes. Much of the existing literature focuses either on finite concept classes or on learning under specific distributional assumptions. For finite classes, 
the extended teaching dimension~\citep{hegedHus1995generalized} quantifies the minimum number of queries required to uniquely identify a target hypothesis; \citet{nowak2008generalized} proposed a condition under which generalized binary search achieves logarithmic query complexity; and \citet{angluin2004queries} provides a broader review of complexity measures for both membership and equivalence queries. However, these results do not directly address the learnability of infinite hypothesis classes. Indeed, every finite concept class can be identified using finitely many queries, making asymptotic convergence rates largely irrelevant in that setting. 

Even for basic yet fundamental infinite classes such as linear halfspaces, comprehensive understanding has not been achieved until recent years. Exponential convergence rates for halfspaces were established  under additional structural assumptions, such as uniform or Gaussian data distributions~\citep{balcan2007margin, diakonikolas2024active}, or homogeneous halfspaces~\citep{chen2017near}. 
On the other hand, one can show exponential learning of halfspaces through the lens of a dual problem named point location, a classical topic in computational geometry~\citep{meyer1984polynomial, meiser1993point, kane2017active}. Most recently, \citet{hopkins2020point} proved a nearly optimal logarithmic query complexity for arbitrary finite unlabeled sets, using isotropic transformations together with margin-based dimensionality reduction.

Finally, we note that there is a substantial literature on enriched query models which  seek to circumvent the limitations of pool-based active learning, while remaining arguably more realistic and interpretable than full membership query synthesis.  Such models include comparison queries~\citep{kane2017active}, equivalence queries~\citep{angluin1988queries}, and region queries~\citep{kontonis2024active}, among others. We cannot do justice to this enriched query literature here, so instead refer the interested reader to \citet{kontonis2024active} and the related work discussion therein.



\newcommand{\passive}{\textrm{passive}}
\newcommand{\pool}{\textrm{pool}}
\newcommand{\syn}{\textrm{syn}}

\section{Background and Preliminaries}\label{sec:prelim}

\subsection{Setup}
Throughout the paper, we consider binary classification in the realizable setting. 
Let $\mathcal X$ denote the input space and let $\mathcal{C} \subseteq \{0,1\}^{\mathcal{X}}$ be a concept (hypothesis) class with finite VC dimension~\citep[see, e.g.,][]{vapnik2013nature}.
We evaluate learning performance with respect to an arbitrary and unknown probability measure $\mu$ on $\mathcal X$. Unlike much of the existing literature on membership query synthesis, which often assumes a specific data distribution, our goal is to understand learnability uniformly over all distributions. We assume the learner has access to a sampling oracle for $\mu$, namely, an oracle that returns independent samples from $\mu$.

\begin{remark}[Access to the distribution]
    Some access to the underlying distribution $\mu$ is necessary for learning. Without either samples from $\mu$ or prior knowledge of $\mu$, an adversarially chosen distribution may place all of its mass outside the queried points, rendering the observations uninformative. A sampling oracle prevents this pathology and allows the learner to approximate the support of $\mu$ through unlabeled samples. In contrast, access to $\mu$ is implicit in passive learning and pool-based active learning, where the observed data points are assumed to be sampled independently from $\mu$.
    \end{remark}

The learner aims to identify an unknown target hypothesis $f\in\mathcal{C}$ from a set of observations $Z_n=\{(X_i, Y_i), i=1,\dots,n\}$
, where each $X_i$ is a query point and $Y_i=f(X_i)$ is its label. One key distinction between different learning models lies in how the query points $X_i$ are generated. We illustrate it via the following three settings, ordered from the least to the most powerful learner: 
\begin{itemize}
    \item \textbf{Passive learning.} $X_i$'s are IID drawn from $\mu$. 

    \item \textbf{Pool-based active learning.} 
    The learner is given an unlabeled pool $S=\{\tilde{X}_1,\dots,\tilde{X}_m\}$ with $\tilde X_i$'s sampled independently from $\mu$. The learner may adaptively choose up to $n$ points in the pool to query.
    
    \item \textbf{Membership query synthesis.} 
    The learner may directly synthesize arbitrary query points in $\mathcal{X}$. Queries are chosen sequentially and adaptively, meaning that each $X_i$ may depend on the entire history $\{X_j,Y_j,j=1,\dots,i-1\}$. 
\end{itemize}

Given the observations $Z_n$, the learner outputs an estimator $\hat f_n \in \mathcal{C}$\footnote{This is known as proper learning. Nevertheless, in Appendix~\ref{sec:app_proper}, we show that allowing improper learning does not bring additional benefit in this setting.}. 
The prediction error of $\hat{f}_n$ is measured by zero-one loss on $\mu$ as $\norm{\hat{f}_n - f} =  \int_{x} \abs{\hat{f}_n(x) - f(x)} \mu(dx)$. Namely, the probability mass on which the estimator and the ground truth disagree.

We use $A$ to denote the entire learning strategy, including both query selection and estimator construction, and let $\mathcal{A}$ denote the class of all possible learning strategies under a given learning model. In particular, we write $\mathcal{A}_{\passive}$, $\mathcal{A}_{\pool}$, and $\mathcal{A}_{\syn}$ for the classes of passive-learning, pool-based active-learning, and membership-query-synthesis strategies, respectively.

    A natural measure of learning difficulty is the minimum number of labeled queries required to learn an unknown target hypothesis. This leads to the standard notion of sample complexity.
    
    \begin{definition}[Sample Complexity]
        For a hypothesis class $\mathcal C$ and a learning model $\mathcal A$, the sample complexity is defined as
        \begin{align*}
            n(\epsilon,\delta;\mathcal{C},\mathcal{A}) \defeq \min \biggl\{n: n\geq 0, \inf_{A \in \mathcal{A}} \sup_{\mu} \sup_{f \in \mathcal{C}} \P(\norm{\hat{f}_n - f} \geq \epsilon) \leq \delta \biggr\}.
        \end{align*}
        If the above set is empty, we define $n(\epsilon,\delta;\mathcal{C},\mathcal{A})=\infty$.
    \end{definition}

    Throughout the paper, we treat $\delta$ as a fixed universal constant and therefore often suppress it from the notation. The asymptotic behavior of $n(\epsilon;\mathcal C,\mathcal A)$ as $\epsilon\to 0$ characterizes the intrinsic difficulty of learning $\mathcal C$ under the learning model $\mathcal A$. In particular, logarithmic sample complexity corresponds to exponentially fast convergence rates, whereas polynomial sample complexity implies substantially slower learning.
    
    To illustrate this viewpoint, we revisit the examples from the introduction.

    \begin{restatement}[continued]{ex:lin}
        For passive learning, we have $n(\epsilon;\mathcal{H}_{\lin}^d,\mathcal{A}_{\passive}) =\Theta(\epsilon^{-1})$. For pool-based active learning, we have $n(\epsilon;\mathcal{H}_{\lin}^1,\mathcal{A}_{\pool}) =\Theta( \ln(1/\epsilon))$ while $n(\epsilon;\mathcal{H}_{\lin}^d,\mathcal{A}_{\pool}) =\Theta( \epsilon^{-1})$ for any $d\geq 2$. 
        In contrast, under membership query synthesis,  $n(\epsilon;\mathcal{H}_{\lin}^2,\mathcal{A}_{\syn}) =\Theta( \ln(1/\epsilon))$, yielding an exponential improvement over the pool-based setting.
    \end{restatement}

    This example reveals three key phenomena:
    \begin{enumerate}
        \item Active learning can substantially reduce sample complexity compared to the passive learning. 
        \item Pool-based active learning exhibits a sharp transition in learnability. While 1D thresholds admit logarithmic sample complexity, higher-dimensional halfspaces require nearly linear complexity.
        \item Membership query synthesis can be strictly more powerful than pool-based active learning. By querying carefully chosen points outside the observed pool, the learner can recover exponentially fast convergence rates even when pool-based methods cannot.
    \end{enumerate}

    \newcommand{\dis}{\textrm{DIS}}
    \newcommand{\starN}{\textsc{s}}
    \subsection{When is exponential speedup possible?}
    The example above suggests that an intriguing distinction among learning models lies in their ability to achieve \emph{exponential speedup}. That is, whether the learner can achieve logarithmic sample complexity or, equivalently, exponentially fast error decay. This leads to the following question:

    \begin{center}
    \emph{Which hypothesis classes admit logarithmic sample complexity in active learning?}
    \end{center}

    For pool-based active learning in the realizable setting, this question was essentially resolved by \citet{hanneke2015minimax}. They introduced a combinatorial quantity called the star number, which completely characterizes when exponential rates are achievable. 
    We briefly restate their definition and result below.

    \begin{definition}[Star number]
        For any set of classifiers $\mathcal{H}$, let $\dis(\mathcal{H}) \defeq \{x \in \mathcal{X}: \exists h,g \in \mathcal{H}, s.t. h(x) \neq g(x) \}$ be the region of disagreement. The star number $\starN(\mathcal{C})$ is the largest integer $s$ such that there exist distinct points $x_1,\dots, x_s \in \mathcal{X}$ and classifiers $h_0,h_1,\dots,h_s \in \mathcal{C}$ such that $\dis(\{h_0,h_i\})\cap \{x_1,\dots,x_s\} = \{x_i\}$ for $i=1,\dots,s$. If no such largest integer exists, define $\starN(\mathcal{C})=\infty$. 
    \end{definition}

    \begin{proposition}[Pool-based learning, Theorem 3 of \protect{\citet{hanneke2015minimax}}]
        Let $\mathcal{C}$ be a class with infinite many hypotheses. 
        For any $\epsilon \in (0, 1/9), \delta \in (0,1/3)$, 
        $n(\epsilon;\mathcal{C},\mathcal{A}_{\pool}) =\Theta( \ln(1/\epsilon))$ when $\starN(\mathcal{C}) < \infty$, otherwise $n(\epsilon;\mathcal{C},\mathcal{A}_{\pool}) =\Omega( \epsilon^{-1})$.
    \end{proposition}    
    
    However, as illustrated by \Autoref{ex:lin}, hypothesis classes with infinite star number may still admit logarithmic sample complexity under membership queries. Consequently, the star number no longer characterizes exponential learnability in this setting. This motivates the main theme of the remainder of the paper: identifying complexity measures that characterize exponential learnability under membership query synthesis.

    At the same time, this observation highlights the potential power of membership queries. In the pool-based setting, only a relatively limited family of hypothesis classes have finite star number and therefore admit exponential speedup. By contrast, membership query synthesis appears to enlarge the collection of classes that can be efficiently learned.

    \begin{remark}[Infinite hypothesis class]
    We mainly focus on infinite hypothesis classes since we are interested in the rate of sample complexity as $\epsilon\to0$. For finite classes, the target hypothesis can always be identified exactly using finitely many queries.
    \end{remark}

    \newcommand{\xtd}{\textrm{XTD}}
    \newcommand{\td}{\textrm{TD}}
    \newcommand{\vs}{\textsc{VS}}
    \newcommand{\atd}{\textrm{STD}}
    \newcommand{\vcd}{VC}
    
\section{Failure of Extending Pool-Based Measures} \label{sec:td}

    As a first step toward understanding membership query synthesis, it is natural to ask whether it is sufficient to extend existing complexity measures from pool-based active learning to this more powerful setting. A particularly appealing candidate is the star number, since it completely characterizes exponential speedup in the pool-based realizable setting.

    At first glance, the star number appears unrelated to the query strategy. However, it is closely connected to the notion of teaching dimension. This connection reveals that the star number is fundamentally tied to the structure of pool-based querying. We briefly review the relevant concepts below.

    Let $\mathcal{C}_{S} = \{ \{(x, f(x)), x \in S\}, f \in \mathcal{C} \}$ denote the collection of labelings on a finite set $S$ induced by $\mathcal C$. 

    \begin{definition}[Version Space~\citep{mitchell1977version}] 
        Let $S\in\mathcal X^m$ and let $h\in \{0,1\}^{\mathcal X}$ be any classifier. The version space of $h$ on $S$ with respect to $\mathcal{C}$ is defined as
        $\vs_{S, h}(\mathcal{C}) \defeq \{f \in \mathcal{C}: f(x) = h(x), x \in S \}$. 
    \end{definition}

    That is, $\vs_{S,h}(\mathcal{C})$ consists of all concepts in $\mathcal C$ that agree with $h$ on $S$.

    \begin{definition}[(Extended) Teaching Dimension Growth Function~\citep{hanneke2007teaching}]
        We define the empirical teaching dimension of any classifier $h$ with respect to a set $S$ and concept class $\mathcal{C}$ as
        \begin{align*}
            \td(h,\mathcal{C},S) = \min\biggl\{t: t\geq 0, \exists U \in S^t \ s.t. \ 
            \abs{\{\vs_{U,h}(\mathcal{C})\}_S}
            \leq 1 \biggr\}.
        \end{align*}
        The teaching dimension growth function is
        \begin{align*}
            \td(\mathcal{C}, m) = \sup_{h \in \mathcal{C}} \sup_{S \in \mathcal{X}^m} \td(h,\mathcal{C},S).
        \end{align*}
        By further allowing $h$ to be an arbitrary classifier outside the concept class,
        the extended teaching dimension growth function is defined as
        \begin{align*}
            \xtd(\mathcal{C}, m) = \sup_{h \in 2^{\mathcal{X}}} \sup_{S \in \mathcal{X}^m} \td(h,\mathcal{C},S).
        \end{align*}
    \end{definition}
    \begin{remark}
        The classical teaching dimension was originally defined for finite concept classes on finite domains~\citep{goldman1995complexity, hegedHus1995generalized}. \citet{hanneke2007teaching} adopts this idea to statistical learning with classifiers on infinite domains. In particular, the empirical teaching dimension above is exactly the teaching dimension of $h$ for a concept class projected on $S$. 
        Although $\td(\mathcal C,m)$ and $\xtd(\mathcal C,m)$ differ from the classical teaching dimension, they play an analogous role for finite projections of an infinite concept class. By a slight abuse of terminology, we therefore refer to them as the (extended) teaching dimension as well. 
    \end{remark}

    \begin{proposition}
    [Equivalence between the star number and teaching dimension]
    \label{prop:star}
        \citet[][Theorem 13]{hanneke2015minimax} proved that, 
        for any $m$, 
        $ \min\{\starN(\mathcal{C}),m\} =TD(\mathcal{C},m) = XTD(\mathcal{C},m) $. 
    \end{proposition}

    Intuitively, the teaching dimension measures the minimum number of labels needed to uniquely determine the labeling of $S$. The corresponding subset $U\subseteq S$ is called a \emph{teaching set}: once the labels on $U$ are revealed, at most one labeling in $\mathcal C_S$ remains realizable. 

    A crucial feature of this definition is that the teaching set must be selected \emph{from the given sample set $S$}. This restriction mirrors the pool-based setting, where the learner may only query labels of points already present in the pool. 

    In contrast, membership synthesis allows querying arbitrary points in $\mathcal{X}$. It is therefore natural to consider a generalization of (extended) teaching dimension in which the teaching set is no longer restricted to lie inside $S$. Specifically, we introduce the following notion of synthesized teaching dimension.

    \begin{definition}[Synthesized Teaching Dimension] 
        For any $h \in 2^{\mathcal{X}}$, we define the empirical synthesized teaching dimension as 
        \begin{align*}
            \atd(h, \mathcal{C}, S) = \min\{t \geq 0: \exists U \in \mathcal{X}^t, \abs{\{\vs_{U,h}(\mathcal{C})\}_S} \leq 1 \}.
        \end{align*}
        Such $U$ is called a teaching set of $h$ on $S$. 
        We further define the synthesized teaching dimension growth function as
        \begin{align*}
            \atd(\mathcal{C}, m) \defeq \sup_{h \in 2^{\mathcal{X}}} \sup_{S \in \mathcal{X}^m} \atd(h, \mathcal{C}, S).
        \end{align*}
    \end{definition}

    Briefly speaking, the synthesized teaching dimension asks how many arbitrarily chosen queries are sufficient to determine the labeling of $S$. Equivalently, $f(x) = h(x)$ on $U$ implies that $f(x) = h(x)$ on $S$.


    At first glance, this quantity appears promising. Indeed, it behaves consistently with several canonical examples, as shown in Proposition~\ref{prop:atd_examples}. The proofs are deferred to Appendix~\ref{sec:app_proof}.

    \begin{proposition}[Examples of Synthesized Teaching Dimension]\label{prop:atd_examples}
        \begin{enumerate}
            \item (Halfspaces have finite STD) We have $\atd(\mathcal{H}_{\lin}^d, m) \leq d\cdot 2^d$ for every sufficiently large $m$. 
            \item (Intervals have infinite STD) Let $\mathcal{C}$ be class of 1D intervals on $[0,1]$, then, $\atd(\mathcal{C}, m) \geq m-1$. 
        \end{enumerate}
    \end{proposition}

    Unfortunately, 
    the following result shows that synthesized teaching dimension is overly optimistic and fails to capture the learning difficulty with membership queries. 

    \begin{theorem}[Counterexample, piecewise linear classifier.]\label{thm:atd}
    There exists a class $\mathcal{C}$ consisting of continuous function with exactly four linear pieces such that $\atd(\mathcal{C}, m) \leq 2$ yet any membership-query algorithm requires polynomial sample complexity.

    \end{theorem}

\section{Conjecture: Halving as the Fundamental Mechanism}\label{sec:half}
    
    The failure of synthesized teaching dimension suggests that simply extending existing pool-based complexity measures is insufficient. We therefore seek a more algorithmic principle underlying exponential learning.

    A key insight from pool-based active learning is that exponential convergence is intimately connected with the ability to efficiently \emph{halve} the version space. In particular, \citet{hegedHus1995generalized} proposed a membership-halving algorithm that halves the version space by querying at most $\xtd(\mathcal{C},m)$ points. Since $\xtd(\mathcal{C},m) = O(1)$ is equivalent to a finite star number (Proposition~\ref{prop:star}), such efficient halving is both sufficient and necessary for exponential convergence in the pool-based setting.

    Crucially, however, this equivalence relies on the restriction that all queries must be selected from the finite set $S$. Once synthesized queries are permitted, it is no longer clear whether the same principle continues to govern learnability. This naturally raises the following question: 
    \emph{Is efficient halving still the fundamental mechanism behind exponential convergence under membership query synthesis?} 

    To investigate this question, we first establish one characterization of exponential convergence under deterministic membership queries.

    \subsection{A characterization via adaptive identification}
 
    Let $\mathcal{A}_t$ denote the collection of all deterministic membership-query-based algorithms that select up to $t$ queries, and let $U=\{X_1,\dots,X_t\}$ be the set of queries. Here, deterministic means that each $X_i$ is uniquely determined by the query history; nonetheless, the selected queries may depend on the unknown target through the previously observed labels. We discuss randomized query strategies in the Appendix~\ref{sec:app_determin}.

    \newcommand{\utd}{\textrm{AIC}}
    \begin{definition}[Adaptive identification complexity (AIC)]
    We define the adaptive identification complexity as
    \begin{align*}
           \utd(\mathcal{C}, m) \defeq \sup_{S \in \mathcal{X}^m} \min\biggl\{t \geq 0: \exists A \in \mathcal{A}_t, s.t.\  \sup_{f \in \mathcal{C}} 
           \abs{\{\vs_{U,f}(\mathcal{C})\}_S}
           \leq 1\biggr\}.
    \end{align*}
    \end{definition}

    Unlike synthesized teaching dimension, which allows a different specifying set for each target hypothesis, $\utd$ requires the existence of a \emph{single adaptive query strategy} that succeeds uniformly for every target $f\in\mathcal C$. $\utd$ can be viewed as an algorithmic analogue of sample complexity, formulated in terms of exact identification on finite samples.

    We can verify that $\utd(m) = O(\ln(m))$ implies exponential learning for VC classes (Appendix~\ref{sec:app_utd}). However, unlike the star number or teaching dimension, \utd \  depends explicitly on the existence of an adaptive querying algorithm and is therefore difficult to compute or analyze directly. This motivates the search for a more structural characterization.
    


    
    \subsection{Halving as a Structural Principle}\label{subsec:halve}
    The above discussion suggests that exponential learning may still arise from repeatedly halving of the candidate space. We therefore introduce the following notions.

    \newcommand{\cd}{\textrm{SHD}}
    \begin{definition}[Synthesized Halving Dimension]
        The synthesized halving dimension of $\mathcal{C}$ is defined as 
        \begin{align*}
        \cd(\mathcal{C}, m) \defeq \sup_{S\in \mathcal{X}^m} \min\biggl\{t \geq 0: \forall V \subseteq \mathcal{C}, \exists U \in \mathcal{X}^t, s.t. \ \sup_{f \in \mathcal{C}} 
        \abs{\{\vs_{U,f}(V)\}_S}
        \leq \lceil \abs{V_S}/2 \rceil \biggr\} .
    \end{align*}
    \end{definition}

    Namely, $\cd(\mathcal{C}, m)$ queries are enough to halve the possible labelings on any $S$ given any version space $V$. Here, $V$ is not necessarily all possible concepts, but allowing refinement, e.g., when we already observed the labeling on some points. We note that the set $U$ above is pre-determined and does not rely on $f$, so \cd \ is a combinatorial complexity like teaching dimension. 
    

    \begin{remark}[Why halving version space instead of unlabeled points]
        That is because halving unlabeled points does not capture the landscape of learning concepts. The labels of points can be correlated, so it is not necessary to halve number of unlabeled points. We provide one such example below. Consider a domain $\{x: x \in \{0,1\}^n\} \cup \{y_1,\dots,y_n\}$ and take $S = \{x: x \in \{0,1\}^n\}$. The hypothesis class is $f_z$, $z\in \{0,1\}^n$ such that $f_z(x) = \ip{x}{z} \ mod \ 2$ and $f_z(y_i)=z_i$, where $z_i$ is the $i$-th coordinate of $z$. Querying $y_i$'s leads to full labeling of $S$, and querying less than $n-1$ points cannot halve $S$. 
    \end{remark}

    It is straightforward to verify that $\cd(m)=O(1)$ guarantees $\utd(m) = O(\ln(m))$ and therefore 
    exponential convergence.
    The converse direction, however, is far less clear. 

    \begin{conjecture} \label{conj:halve}
    If $\utd(m) = O(\ln(m))$, then $\cd(m)=O(1)$.
    \end{conjecture}

    The conjecture is motivated by an information-theoretic intuition. Since each binary query reveals at most one bit of information, $O(\ln(m))$ queries are essentially optimal. Thus, if $\utd(m) = O(\ln(m))$, some queries along the interaction must substantially reduce the remaining uncertainty. Moreover, one would expect such informative queries to occur regularly, say, once every constant number of steps. Intuitively, it is difficult to imagine that the first highly informative query could be delayed arbitrarily far as $m$ grows. 

    The difficulty lies in making this intuition rigorous. Although $O(\ln m)$ total queries suffice, it is unclear whether a constant-factor reduction in uncertainty must occur within constantly many steps. In principle, there may exist a sequence of increasingly difficult instances for which the early queries reveal very little information, while only the later queries become highly informative. As we shall see shortly, this phenomenon can occur for finite hypothesis classes. 

    At present, it remains unclear whether such a phenomenon can occur for infinite hypothesis classes. Resolving Conjecture~\ref{conj:halve} would therefore lead to one of two fundamentally different conclusions.

    (1) There exists an infinite hypothesis class such that $\cd(m) = \utd(m) = O(\ln(m))$. This would demonstrate that efficient halving is not necessary for exponential convergence under membership query synthesis, revealing a learning mechanism fundamentally different from that in pool-based active learning.

    (2) Alternatively, one may prove Conjecture~\ref{conj:halve}. In this case, efficient halving would remain the fundamental mechanism underlying exponential convergence.

    \textbf{Finite classes that are hard to halve initially.}
    We construct a family of finite concept classes for which $\utd$ is logarithmic, while the synthesized halving dimension is also logarithmic.

    For any integer $n\ge 1$, let $\mathcal{I}=\{1,\ldots,n\}$ be an index set, $\mathcal{Q} = \{q_1,\dots, q_{2^n}\}$, $\mathcal{Q}^i, i=1,\dots,n$ be independent copies of $\mathcal{Q}$, whose elements are denoted as $q^i_j, j=1,\dots,2^n$, and $\mathcal{V}=\{v_j^{i}, i=1,\dots,n,j=1,\dots,n\}$. 

    We define a concept class over the domain $\mathcal{X} = \mathcal{I} \cup \mathcal{V} \cup \bigcup_{i=1}^n \mathcal{Q}^i$. Each hypothesis is index by $i^* \in \{1, \dots, n\}$ and $j \in \{0, \dots, 2^n-1\}$:
    \begin{align*}
        f_{i^*,j^*}(i) &= \ind_{i=i^*}, i \in \mathcal{I};
        \\ f_{i^*,j^*}(v_j^{i}) &= \ind_{i=i^*} \wedge z^*_{j}, v_j^{i} \in \mathcal{V};
        \\ f_{i^*,j^*}(q^i_j) &= \ind_{i=i^*} \wedge \ind_{j=j^*}, q^i_j \in \mathcal{Q}^i,
    \end{align*}
    where $z^* \in \{0,1\}^n$ be the binary representation of $j^*$. 

    Intuitively, each hypothesis selects a distinguished tree indexed by $i^*$. All nodes in the remaining trees receive label $0$, while the selected tree contains a unique positive leaf indexed by $j^*$. The bits of $j^*$ are encoded on the auxiliary vertices $\mathcal{V}$. The resulting concept class contains $m\defeq n2^n$ hypothesis. We visualize the concept class in \Autoref{fig:tree} in the Appendix~\ref{sec:app_fig}. 

    Consider the unlabeled sample $S=\bigcup_{i=1}^n \mathcal{Q}^i$. Clearly, $\abs{S}=m$. An active learner can first query all points in $\mathcal{I}$ to identify $i^*$, requiring at most $n$ queries. It can then query all vertices in $\{v^{i^*}_j,j=1,\dots,n\}$ to recover the binary representation $z^*$ and hence determine $j^*$, requiring another $n$ queries. Consequently, $\utd(m)=2n=O(\ln(m))$. On the other hand, to eliminate even half of the candidate hypotheses, the learner must eliminate half of the index set. Thus, any strategy requires at least $n/2$ queries in the worst case before the version space can be halved, leading to $\cd(m) \geq n/2 = \Omega(\ln(m))$.

    Since this construction can be carried out for every $n$, it yields an infinite sequence of finite concept classes that are easy to identify globally, yet difficult to halve at the initial stages. 
    One possible route toward resolving Conjecture~\ref{conj:halve} is to investigate whether an analogous separation exists for infinite concept classes, or whether such a separation is fundamentally impossible. ~\footnote{One might suspect that the disjoint union of the above class over all $n$ would serve to prove Conjecture~\ref{conj:halve}. However, the resulting infinite class would not admit $AIC(m) = O(\log m)$. Details are in Appendix~\ref{sec:app_union}.}

    \textbf{Classes admitting constant synthesized halving dimension.} As a weaker version of Conjecture~\ref{conj:halve}, we proved the following result. 

    \begin{theorem}[Near-optimal AIC implies constant halving]\label{thm:halve_near}
        If there exists a constant $B>0$, such that $\utd(\mathcal{C},m) \leq \log(m) + B$ for every $m$, then we have 
        $\cd(\mathcal{C},m) \leq 2^{B+2}-1$ for every $m$. 
    \end{theorem}

    \Autoref{thm:halve_near} indicates that a nearly optimal $\utd$ does imply a constant $\cd$. Intuitively, this is because the assumed query complexity matches the information-theoretic lower bound in both order and leading constant. Consequently, only a constant number of queries can be substantially unbalanced; otherwise, the accumulated information loss would force the total query complexity above the assumed bound. Thus, the corresponding decision tree must be nearly balanced and have logarithmic depth. That said, a nontrivial gap remains between $\utd(m) = \log(m)+B$ and $\utd=\Omega(\log(m))$, which we leave as future work.

    We conclude this section by discussing a possible structural property that may lead to a constant synthesized halving dimension. 
    Many commonly studied concept classes seem to exhibit a form of \emph{homogeneity}: restricting attention to a local region of the instance space neither significantly simplifies nor complicates the learning problem. Intuitively, for such classes, a constant number of appropriately chosen queries may suffice to reduce the remaining uncertainty by a constant factor, potentially implying $\cd(m)=O(1)$. Formalizing this notion of homogeneity and identifying the conditions under which it indeed guarantees a constant synthesized halving dimension remain interesting open questions.

\section{Conclusions and Further Discussions} \label{sec:con}

This work has been motivated by the observation that active learning with membership query synthesis exhibits fundamentally different behavior from the pool-based setting. In particular, classical complexity measures such as the teaching dimension fail to characterize when exponentially fast learning is possible. Motivated by this gap, we conjectured that the ability to efficiently halve the version space remains the key structural property underlying exponential convergence, and proposed the synthesized halving dimension as a candidate complexity measure.

Several important questions remain open. Foremost is resolving Conjecture~\ref{conj:halve}: does logarithmic $\utd$ imply constant $\cd$? A positive answer would establish efficient halving as the fundamental mechanism underlying exponential active learning, while a negative answer would reveal a new learning principle unique to synthesized queries. Another important direction is to extend our analysis beyond the realizable setting and understand the role of synthesized queries in the presence of label noise. 

\section*{AI Disclosure}
    The authors used ChatGPT for language editing to improve the clarity and readability of the manuscript. ChatGPT was also used to help inspire the development of \Autoref{thm:atd, thm:halve_near}. All other results and proofs were developed independently by the authors.

\section*{Acknowledgements}
The authors would like to thank Haifeng Xu, Idan Attias, and Max Hopkins for helpful discussions.

\newpage
\bibliography{ref}
\bibliographystyle{plainnat}

\newpage
\appendix

\section{Additional Figures}\label{sec:app_fig}
\begin{figure}[tbh]
    \centering
    \includegraphics[width=\linewidth]{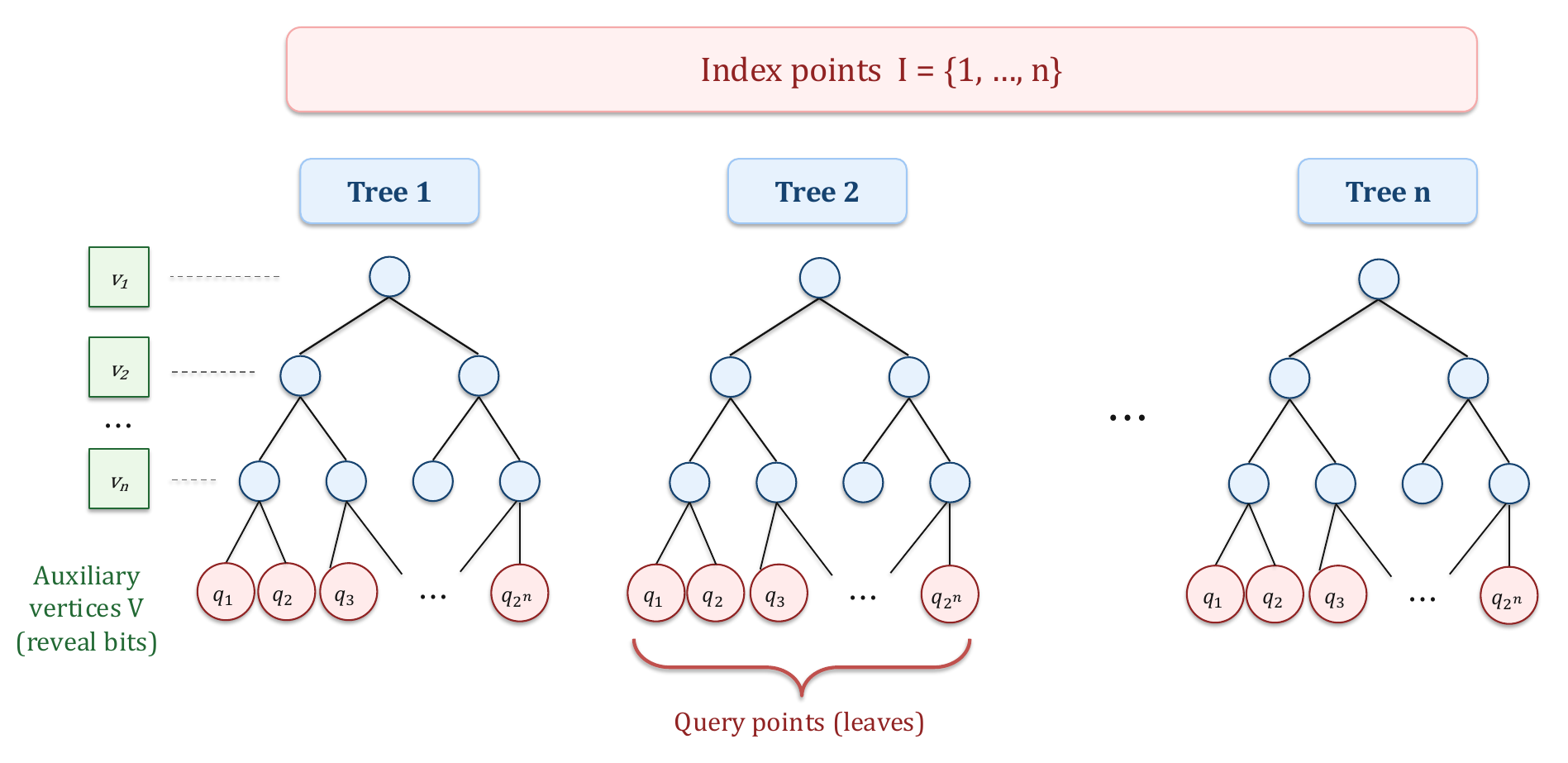}
    \caption{Illustration of the finite concept class. Each index $i\in\mathcal I$ corresponds to a balanced tree with $2^n$ leaves $\mathcal Q^i$. A hypothesis $f_{i^*,j^*}$ selects one tree $i^*$ and one positive leaf $q_{j^*}^{i^*}$, while the auxiliary vertices $v_1^{i^*},\ldots,v_n^{i^*}$ encode the $n$ bits of the leaf index $j^*$.}
    \label{fig:tree}
\end{figure}

\section{Exponential Convergence for Halfplanes}\label{sec:lin_syn}
    In this section, we present a direct geometric proof of exponential learnability for two-dimensional halfspaces under membership query synthesis. More general results for arbitrary dimensions have been established through connections to the classical point-location problem in computational geometry~\citep{meiser1993point} and, more recently, through isotropic transformations and margin-based dimensionality reduction~\citep{hopkins2020point}. In contrast, our proof offers a simple and intuitive argument based on centerpoints and convexity.
    
    The key idea is to exploit the geometric structure of $\real^2$. Specifically, we construct a querying strategy that identifies the labels of a constant fraction of an arbitrary finite set using only a constant number of synthesized queries. Repeating this procedure yields exponentially fast convergence.
    
    \begin{theorem}\label{thm:2d}
        For any distribution $\mu$ on $[0,1]^2$, and any $f \in \mathcal{H}_{\lin}^2$, 
        there exists an active learning algorithm $A$ and a constant $C$ such that 
        \begin{align*}
            \norm{\hat{f}_n-f} \leq e^{-Cn}.
        \end{align*}
        Equivalently, for any $\epsilon > 0$, at most $O(\ln(1/\epsilon))$ queries are needed to achieve $\norm{\hat{f}_n-f}_2^2 \leq \epsilon$.
    \end{theorem}

    \textbf{Proof Sketch.} The proof proceeds in three steps. First, we reduce the problem to exactly recovering the labels on a finite set $S$. Second, we use the centerpoint theorem~\citep{rado1946theorem} to identify a suitable coordinate system in which each quadrant contains a constant fraction of the points in $S$. Finally, we design an adaptive querying strategy that, using only a constant number of membership queries, certifies the labels of a constant fraction of the remaining points. Iterating this procedure leads to logarithmic query complexity. The complete proof is included in Appendix~\ref{sec:app_proof}.

    The proof relies on two geometric properties specific to the plane. First, the centerpoint theorem guarantees the existence of a ``center point'' such that every halfspace contains it also contains a constant fraction of the data. Second, the 2D geometry allows us to choose coordinate axes so that the convex hulls generated by queried points eliminate a constant fraction of the remaining uncertainty. Hence, it remains unclear whether this proof can be extended to higher dimensions, where constructing informative convex hulls becomes considerably more challenging. We leave this question as an interesting direction for future work.

\section{Missing Proofs}\label{sec:app_proof}

\textbf{Notations. }
We denote two sequences $a_n$ and $b_n$ at the same order as $a_n = \Theta(b_n)$, namely, $C_1 \leq \lim_{n \to \infty} a_n/b_n \leq C_2$ for some constants $C_1, C_2$. 
We also write $a_n = O(b_n)$ if $\lim_{n \to \infty} a_n/b_n \leq C$, and $a_n = \Omega(b_n)$ if $\lim_{n \to \infty} a_n/b_n \geq C$. 
Let $\ind$ denote the indicator function. 

\begin{proofof}{\Autoref{prop:atd_examples}.} We first consider the class of halfspaces.
    For $d=1$, it is straightforward to verify that $\atd(\mathcal{H}_{\lin}^1, m) = 2$ for $m\geq 2$. Next, we consider the case $d\geq 2$. 

        \textbf{Case 1:} $\vs_{S,h}(\mathcal{H}_{\lin}^d) =\emptyset$. In other words, $h$ is not realizable by a halfspace on $S$, meaning that the positive and negative points in $S$ are not linearly separable.
        By the hyperplane separation theorem, the convex hulls of the positive and negative examples intersect. By Carathéodory's theorem, this intersection can be witnessed by at most $d+2$ points. These $d+2$ points consist a specifying set $U$ as desired. 

        \textbf{Case 2:} $\vs_{S,h}(\mathcal{H}_{\lin}^d) \neq \emptyset$. Now, there exists a hyperplane $T: x^\T\beta + a = 0$ that achieves perfect separation on $S$ given $h$. 
        When $T$ does not intersect with $[0,1]^d$, then all points in $[0,1]^d$ receive the same label. As a result, a specifying set $U$ consists of all vertices of $[0,1]^d$. Otherwise, $T$ intersects with $[0,1]^d$. Then, one can always find a constant $\eta_1 \geq 0$ such that the halfspace $x^\T\beta + a + \eta_1 \geq 0$ includes all positive points in $S$. Similarly, let $\eta_2 < 0$ be a constant such that the halfspace $x^\T\beta + a + \eta_2 < 0$ includes all negative points in $S$. 
        A specifying set $U$ is intersection points between $[0,1]^d$ and two hyperplanes ($x^\T\beta + a + \eta_j = 0, j=1,2$). There are at most $d\cdot 2^{d}$ vertices because there are at most $d\cdot 2^{d-1}$ edges for $[0,1]^d$. 
        
    We next consider 1D intervals. For any $m \geq 0$, let $f_i = \ind_{x = i/m}, i=1,\dots, m$, $h=0$, and $\mu$ be a uniform distribution on $S \defeq \{i/m, i=1,\dots, m\}$. For any $U$ such that $\abs{U}<m-1$, it is clear that at least two points in $S$ are not included in $U$, therefore $\abs{(V_{U, h})_{\mathcal{C}} \cap \mathcal{C}/S} \geq 2$. It indicates that $\atd(\mathcal{C}, m) \geq m-1$. We thus complete the proof.

\end{proofof}

\begin{proofof}{\Autoref{thm:atd}.} \textbf{Construction of function class.} We partition the positive integers into consecutive blocks $B_k$, $k\geq2$, of size $L_k=k2^k$. Within $B_k$, assign bump numbers $1,\ldots,k$ cyclically, repeated $2^k$ times, and denote $A_{k,i}$ for the indices assigned bump $i$. Thus, each integer $n$ corresponds to a unique pair $(k,i)$. 

We associate each function $f_n$ in $\mathcal{C}$ with a bump function and a unique base height. For each pair $(k,i)$, $i=1,\dots,k, k=1,2,\dots$, choose an interval $J_{k,i}=(a_{k,i},b_{k,i})$, with all their closures pairwise disjoint and contained in $(0,1)$. Let $c_{k,i}=(a_{k,i}+b_{k,i})/2$ be the midpoint of $J_{k,i}$, and let $\tau_{k,i}$ be the triangular bump of height $1/2$ at $c_{k,i}$, zero outside $J_{k,i}$. For $n\in A_{k,i}$ define
\[
 \beta_n=\frac14-\frac1{8n},\qquad g_n(x)=\beta_n+\tau_{k,i}(x),
 \qquad f_n(x,y)=\ind\{y\leq g_n(x)\}, x,y\in(0,1).
\]
Set $\mathcal{C}=\{f_n:n\geq1\}$. Namely, our function class consists of epigraph of boundary function $g_n(x)$, which is a bump function with baseline height $\beta_n$. The four graph pieces are flat, rising, falling, flat. We demonstrate $f_{n}$ with $n=54$ in \Autoref{fig:bump}

\begin{center}
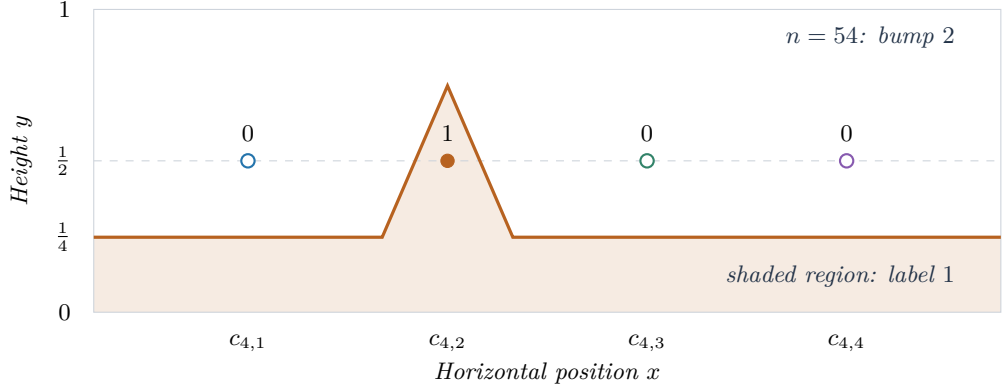

\begin{tikzpicture}[x=12cm,y=4.0cm,font=\small]
\pgfmathsetmacro{\base}{.25-1/(8*54)}
\fill[bumptwo!13] (0,0)--(0,\base)--(.318,\base)--(.39,{\base+.5})--(.462,\base)--(1,\base)--(1,0)--cycle;
\draw[grid] (0,0) rectangle (1,1);
\draw[grid,dashed] (0,.5)--(1,.5);
\draw[bumptwo,very thick] (0,\base)--(.318,\base)--(.39,{\base+.5})--(.462,\base)--(1,\base);
\foreach \xx/\cc/\ii/\lab in {.17/bumpone/1/0,.39/bumptwo/2/1,.61/bumpthree/3/0,.83/bumpfour/4/0}{
  \ifnum\ii=2 \filldraw[fill=\cc,draw=\cc] (\xx,.5) circle[radius=2.5pt];
  \else \filldraw[fill=white,draw=\cc,thick] (\xx,.5) circle[radius=2.5pt];\fi
  \node[above=4pt] at (\xx,.5) {$\lab$};
  \node[below=5pt] at (\xx,0) {$c_{4,\ii}$};
}
\foreach \yy/\lab in {0/0,.25/\frac14,.5/\frac12,1/1}
 \node[left=5pt] at (0,\yy) {$\lab$};
\node[rotate=90] at (-.08,.5) {Height $y$};
\node at (.5,-.20) {Horizontal position $x$};
\node[anchor=east,text=ink] at (.96,.91) {$n=54$: bump $2$};
\node[anchor=east,text=ink] at (.96,.12) {shaded region: label $1$};
\end{tikzpicture}
\captionof{figure}{One hypothesis in block $B_4=\{33,\ldots,96\}$. The four test points at height $1/2$ have labels $0,1,0,0$. The symbol $c_{4,i}$ is the horizontal coordinate of bump $i$'s peak. As $n$ increases, the bump cycles through its four positions while the baseline rises by tiny amounts.}\label{fig:bump}
\end{center}

\textbf{Finite $\atd$.} Briefly speaking, distinct baseline height supply the ordered markers as defined below. If an arbitrary $h$ imitates a finite index, two markers suffice. 
Define the midpoint markers
\[
 z_t=\left(0,\frac{\beta_t+\beta_{t+1}}2\right),\qquad
 f_n(z_t)=\ind\{n>t\},\quad t\geq1.
\]
A teacher knowing $n=54$ can query $z_{53}$ and $z_{54}$. Their labels $1,0$ force $n\geq54$ and $n\leq54$, respectively, as shown in \Autoref{fig:marker}. 

\begin{center}
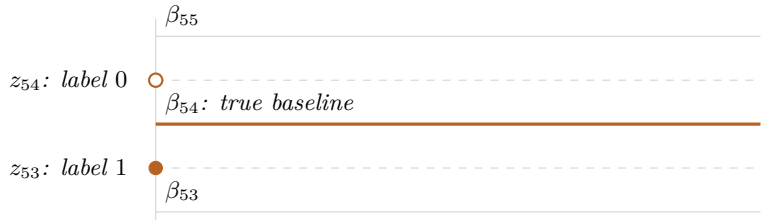

\begin{tikzpicture}[x=1cm,y=.58cm,font=\small]
\draw[grid] (0,-.3)--(0,4.4);
\foreach \yy/\nn in {0/53,4/55}{
 \draw[grid] (0,\yy)--(8,\yy);
 \node[above right] at (0,\yy) {$\beta_{\nn}$};
}
\draw[bumptwo,very thick] (0,2)--(8,2);
\node[above right] at (0,2) {$\beta_{54}$: true baseline};
\foreach \yy/\tt/\ans in {1/53/1,3/54/0}{
 \draw[grid,dashed] (0,\yy)--(8,\yy);
 \node[left=7pt] at (0,\yy) {$z_{\tt}$: label $\ans$};
}
\filldraw[bumptwo] (0,1) circle[radius=2.5pt];
\filldraw[draw=bumptwo,fill=white,thick] (0,3) circle[radius=2.5pt];
\end{tikzpicture}
\captionof{figure}{At $x=0$, two marker queries bracket the true baseline. Vertical gaps are enlarged for visibility.}\label{fig:marker}
\end{center}

Fix $S$ and $h \in 2^{\mathcal{X}}$, there are exactly three cases for its infinite marker sequence:
\begin{enumerate}
\item \textbf{A zero followed by a one.} If $h(z_r)=0$ and $h(z_s)=1$ for $r<s$, no index can satisfy both $n\leq r$ and $n>s$. These two queries leave no hypothesis, so $\atd(h,\mathcal{C},m)=2$.
\item \textbf{A finite first zero.} Otherwise, if the first zero is at $n$, the answers at $z_{n-1},z_n$ leave precisely $f_n$. For $n=1$, $z_1$ alone suffices. This is valid even if $h$ disagrees with $f_n$ elsewhere: a singleton has only one restriction to $S$.
\item \textbf{All marker answers are one.} Every fixed $x$ lies in at most one bump interval, and that bump is used for only finitely many indices. After those indices, $f_n(x,y)=\ind\{y\leq\beta_n\}$, which converges to $\ind\{y<1/4\}$. Consequently, for some $T$, \emph{all} $n>T$ agree on the finite set $S$. Querying $z_T$, whose $h$-label is $1$, leaves exactly this tail and hence one restriction to $S$.
\end{enumerate}
The above discussion proves that $\atd(\mathcal{C},m)\leq 2$. 

\textbf{Learning is still hard.} 
We consider the following unlabeled points 
\[
 S_k=\{(c_{k,i},1/2):1\leq i\leq k\}.
\]
For $n\in B_k$, exactly one point in $S_k$ is positive, and its coordinate identifies the bump number. Test points illustrated in \Autoref{fig:bump} represent $S_k$. 

Now, we prove the difficulty in identifying the bump number of $f_n$. Let $D_k$ be a uniform distribution on $S_k$. So an algorithm cannot label all points correctly has error at least $1/k$. 
We introduce two types of auxiliary queries:
\[
 Q_t(n)=\ind\{n>t\},\qquad P_{b,i}(n)=\ind\{n\in A_{b,i}\}.
\]
The task is equivalent to correctly identify the bump number of $n$, given that $n$ is uniform on $B_k$.

For general $k$, there are $k2^k$ indices in $B_k$. Testing bumps individually takes order $k$ queries; binary-searching the index also takes order $\log(k2^k)=\Theta(k)$. Lemma~\ref{lem:auxiliary} shows that mixing the two query types cannot beat order $k$ either. 

Furthermore, every query $(x,y)$ can be simulated with at most two auxiliary queries. For a query $(x,y)$, at most one bump can be nonzero at $x$. Query its $P_{b,i}$, if needed, to determine whether it is active. Then compare $\beta_n$ with $y$ minus the now-known bump value. Since $\beta_n$ increases strictly, this comparison is a $Q_t$ query or a constant. Thus, every query can be recovered by at most two auxiliary queries.

Combining the above two observations, $O(k)$ queries are necessary to achieve an error below $1/k$, which completes the proof.  
\end{proofof}

        \newcommand{\disagr}{\texttt{DIS}}
    \begin{proofof}{Theorem~\ref{thm:2d}.}
        We repeatedly use the following elementary property of linear classifiers: the convex hull of positively labeled points must also be labeled positive, and similarly for negatively labeled points. 
        
        \textbf{Step 1: Reduction to a finite set.} 
        Let $\disagr(f,g) \defeq \{x \in \mathcal{X}: f(x) \neq g(x)\}$ be the disagreement set of two concepts $f,g \in\mathcal{C}$. 
        We define the following pseudo-metric on $\mathcal{C}$: $D(f,g) \defeq \mu(\disagr(f,g)).$
        By Lemma~\ref{lemma:net}, we can sample an $\epsilon$-net $S$ of $\{D(f,g), f,g \in\mathcal{C}\}$ with sample size $\abs{S} \defeq m = 3c'\epsilon^{-1}\ln(1/\epsilon)$. Namely, if $f(x) = g(x)$ for all $x \in S$, we have $D(f,g) \leq \epsilon$.
        
        Therefore, by choosing $\delta=\epsilon$, it suffices to exactly recover the labels on $S$ with $O(\ln(m))$ queries. In particular, we construct an algorithm that, in each round, determines the labels of at least a constant fraction of the remaining unlabeled points using only a constant number of queries.

        \textbf{Step 2: Querying the corners.} 
        The query procedure is illustrated in \Autoref{fig:alg}.
        Let $c$ be a centerpoint of $S$ (Lemma~\ref{lemma:center}). By translating and rescaling the space, we may assume without loss of generality (WLOG) that $c$ is the origin and $S \in [-1,1]^2$.
        Lemma~\ref{lemma:center} indicates that any halfspace containing $c$ contains at least $\lceil m/3 \rceil$ points in $S$. 
        
        We first query the four corners of the square $[-1, 1]^2$. If three or more corners share the same label, then the entire convex hull of these corners must receive the same label. Since this convex hull contains the centerpoint $c$, we eliminate $m/3$ points of $S$. 

        The only remaining possibility is that the two upper corners have one label and the two lower corners have the other. WLOG, assume that the upper corners have label one, while the lower concerns are zeros. We then proceed to the next step.

        \textbf{Step 3: Querying the center and edge midpoints.} We query the label of $c$, which WLOG assumed to be zero. Next, we query the two edge midpoints $(-1, 0)$ and $(1, 0)$. If both of them are zeros, then again we eliminate $\lceil m/3 \rceil$ points. 

        In the case where the two midpoints disagree, WLOG, we assume $(-1, 0)$ has label one and $(1, 0)$ is zero. Note that if the convex hull of zeros $H_0$ have more than $\rho m$ points, we are done. 

        It remains to show that $\abs{S \cap H_0} \geq \rho m$ for a constant $\rho$. 
        In fact, after an appropriate rotation of coordinates, Lemma~\ref{lemma:center} guarantees that each quadrant contains at least $m/6$ points.
        To see it, suppose one of the quadrants, say $H_1$, has less than $m/6$ points. Since every halfspace containing the centerpoint contains at least $m/3$ points, we may rotate the coordinate system and choose a second line $v$ through the centerpoint so that the sector $H_1\cap H_2$ contains exactly $m/6$ points. Taking $y$-axis and $v$ as the new coordinate axes yields a partition of the plane into four quadrants, each containing at least $m/6$ points. In the special case where more than $m/6$ points exactly lay on the $y$-axis, one can efficiently eliminate half of them by bisection method.


        As a result, in each round, Step 2 and 3 eliminate at least $m/12$ points, at the cost of querying at most $7$ points. We thus complete the proof.

        
        
    \end{proofof}

    \begin{figure}
    \begin{minipage}[c]{0.48\textwidth}
        \centering
        \includegraphics[width=\linewidth]{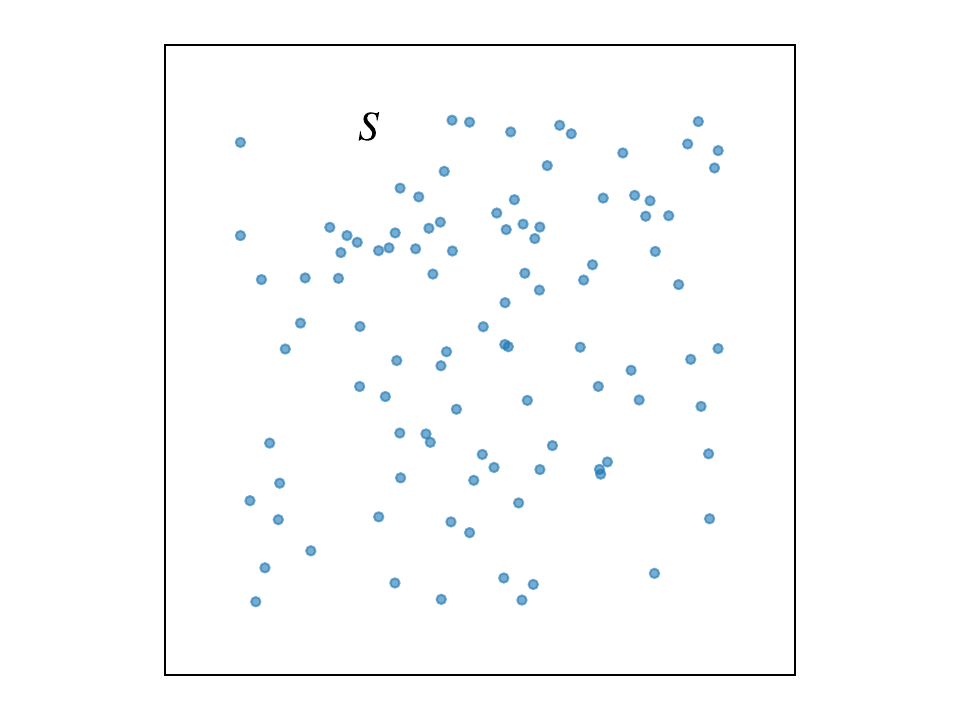}
        {Step 1: Find an $\epsilon$-net $S$}
    \end{minipage}
    \hfill
    \begin{minipage}[c]{0.48\textwidth}
        \centering
        \includegraphics[width=\linewidth]{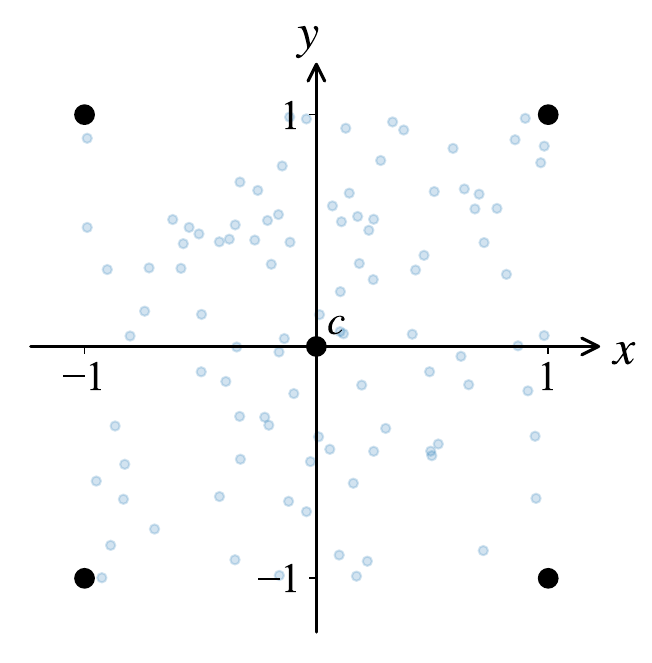}
        {Step 2: Query the corners}
    \end{minipage}
    \vspace{0.5cm}
    \begin{minipage}[c]{0.48\textwidth}
        \centering
        \includegraphics[width=\linewidth]{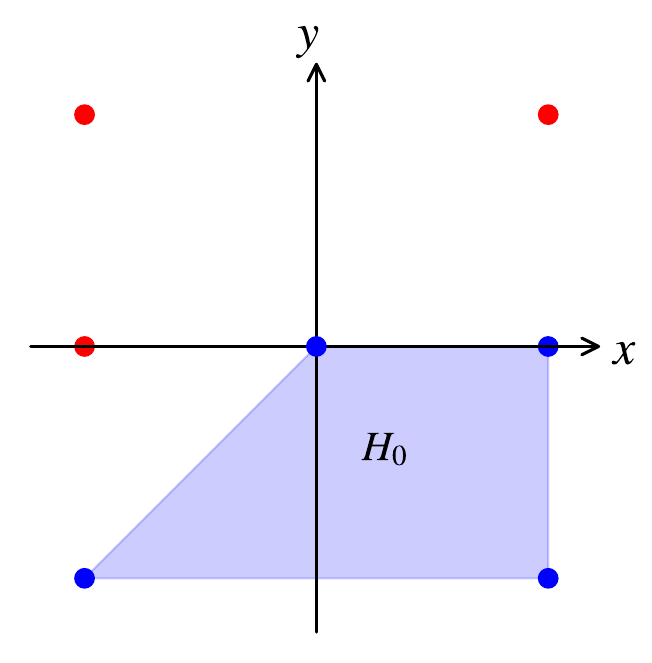}
        {Step 3: Query the centerpoint and two points on the edge}
    \end{minipage}
    \hfill
    \begin{minipage}[c]{0.48\textwidth}
        \centering
        \includegraphics[width=\linewidth]{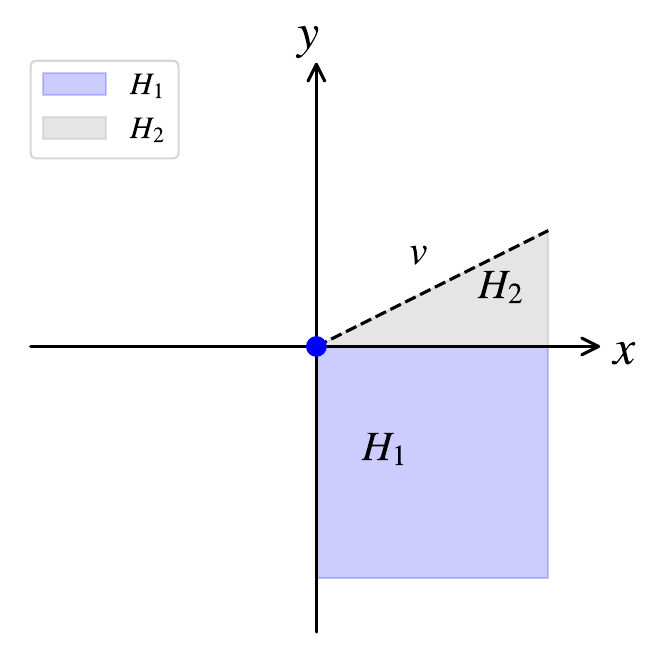}
        {There exists a transformation ensures that every quadrant has sufficient many points.}
    \end{minipage}
    \caption{Demonstration of the proposed membership query strategy that eliminates a fraction of unlabeled points with constant number of queries.}
    \label{fig:alg}
    \end{figure}

    \begin{proofof}{\Autoref{thm:halve_near}.}
        For any fixed $S$ and $V$, let $N=\abs{V_S}$, and $D$ be the minimum depth of an adaptive binary decision tree induced by the membership query strategy that exactly identifies the labeling on $S$ for every $f \in V$. Further denote $r= \max\{0, D-\lfloor \log(N/2) \rfloor\}$. 

        Let $U$ be the union of all possible queries appearing in the first $r$-level of this tree. We have $\abs{U} \leq \sum_{i=1}^r 2^{i-1} = 2^r-1$. 
        Now, for any labeling on $S$, since it will be revealed after $D$ queries, the subtree of any node in the $r$-th level has depth at most $D-r\leq \lfloor \log(N/2) \rfloor$. Therefore, the remaining possible labelings after querying the first $r$ queries are at most $2^{D-r} \leq N/2$. In other words, $2^r-1$ queries is sufficient to halve $V_S$.

        Note that any $N$ distinct labelings can be separated by at most $N-1$ points of $S$. We have $D \leq \utd(V,N) \leq \log(N)+B$ and $r\leq B+2$, since $\utd(\mathcal{C}, m) \leq \log(m)+B$ for every $m$. 
        We conclude the proof as $V$ and $S$ can be chosen arbitrarily and $r$ is independent of them. 
    \end{proofof}

\section{Technical Lemmas}

    \begin{lemma}[\protect{\citet[Lemma 21]{hanneke2015minimax}}]\label{lemma:net}
        Suppose $\mathcal{C}$ is a class with VC dimension $d$ and an oracle sampler is available. 
        There exists a universal constant $c' \geq 1$ such that, for any distribution $\mu$ on $\mathcal{X}$ and $\delta > 0$, one can sample a set $S=\{x_1,\dots,x_m\}$ that is an $\epsilon$-net of $\mu$ for $\{\dis(\{h,g\}): h,g \in \mathcal{C}\}$ with probability at least $1-\delta$. Moreover, $\abs{S} = \lceil c'd\epsilon^{-1}\ln(1/\epsilon) \rceil$.
    \end{lemma}

    \begin{lemma}[Centerpoint Theorem\protect{~\citep{rado1946theorem}}]\label{lemma:center}
        Let $S \subset \mathbb{R}^d$ be a finite set of $m$ points. Then there exists a point $c \in \mathbb{R}^d$ such that every closed halfspace containing $c$ contains at least $\lceil m/(d+1) \rceil$ points of $S$.
The point $c$ is called a centerpoint of $S$.
    \end{lemma}

    \begin{lemma}\label{lem:auxiliary}
    Consider the function class in the proof of \Autoref{thm:atd}. 
Fix $k\geq 2$, and let $N$ be uniformly distributed over the block $B_k$. Let $\rho_k(N)\in[k]$ denote the bump number associated with $N$. Suppose an algorithm may adaptively make the following two types of auxiliary queries:
\begin{align*}
    Q_t(N) = \ind\{N>t\},\quad
    P_{b,i}(N) = \ind\{N\in A_{b,i}\}.
\end{align*}
Then any possibly randomized algorithm making at most $t$ queries and outputting
an estimate $\widehat \rho$ of $\rho_k(N)$ satisfies
\begin{align*}
    \P\bigl(\widehat \rho=\rho_k(N)\bigr)
    \leq
    \frac{t+1}{k}
    +
    \frac{2^{t+1}-1}{k2^k}.
\end{align*}
Consequently, for any $k\geq8$ and $t\leq k/4$, the estimation accuracy is less than $1/2$.
\end{lemma}

\begin{proofof}{Lemma~\ref{lem:auxiliary}.}
We first fix the internal randomness of the algorithm, so that the resulting query strategy is deterministic yet adaptive: at each node of its decision tree, its next query is a deterministic function of the query-answer history along the path to that node. In particular, fixing $R$ does not make the sequence of queries nonadaptive. 
We then prove the desired bound uniformly for every such deterministic adaptive decision tree. Averaging the resulting bound establishes the claim for randomized learners. 

Since $N\in B_k$, any query $P_{b,i}$ with $b\neq k$ always returns zero and provides no information. We may therefore restrict attention to queries $P_{k,i}$.

Consider the depth-$t$ decision tree induced by the algorithm. Whenever a $P_{k,i}$ query receives answer $1$, we terminate that branch and count it as a success, regardless of the subsequent behavior of the algorithm. This can only increase the probability of success. Hence, on every branch that remains active, all preceding $P$-queries have returned zero.

We next consider a path where $P$-queries are all negative before depth $j$. 
Since every $Q_t$ is a threshold query of the form $N>t$, the set of indices satisfying a fixed sequence of $Q$-answers is an integer interval. Therefore, at depth $j$, the intervals associated with different active nodes are pairwise disjoint, and there are at most $2^j$ of them. Denote these intervals by $I_1,\dots,I_r$, where $r\leq 2^j$. Then
\begin{align*}
    \sum_{\ell=1}^r |I_\ell|
    \leq \abs{B_k}=L_k=k2^k.
\end{align*}

By construction, the bump numbers $1,\dots,k$ occur cyclically within $B_k$. Consequently, for every integer interval $I\subseteq B_k$ and every $i\in[k]$,
\begin{align}\label{eq:cyclic-count}
    \bigl|\{n\in I:\rho_k(n)=i\}\bigr|
    \leq \frac{|I|}{k}+1.
\end{align}


We can now bound the probability that the algorithm encounters a positive $P$-answer. Consider depth $j<t$. At each active node associated with an interval $I_\ell$, the algorithm can ask at most one query $P_{k,i_\ell}$. By \Autoref{eq:cyclic-count}, at most $1+|I_\ell|/k$ indices in $I_\ell$ can produce a positive answer to this query. Therefore, the total number of indices that can produce their first positive $P$-answer at depth $j$ is at most
\begin{align*}
    \sum_{\ell=1}^r
    \left(\frac{|I_\ell|}{k}+1\right)
    &\leq
    \frac{L_k}{k}+2^j.
\end{align*}
Since $N$ is uniform on $B_k$, the probability of a first positive $P$-answer
at depth $j$ is at most
\begin{align*}
    \frac{1}{k}+\frac{2^j}{L_k}.
\end{align*}
Summing over $j=0,\dots,t-1$, the total probability of all branches terminated
by a positive $P$-answer is at most
\begin{align}\label{eq:positive-P}
    \sum_{j=0}^{t-1}
    \left(
        \frac{1}{k}+\frac{2^j}{L_k}
    \right)
    =
    \frac{t}{k}
    +
    \frac{2^t-1}{L_k}.
\end{align}

It remains to consider branches on which every $P$-query returns zero. There are at most $2^t$ terminal leaves. 
Each such leaf is associated with an integer interval $I_\ell\subseteq B_k$, and these intervals are pairwise disjoint. At a terminal leaf, the algorithm outputs a single bump number $\widehat\rho_\ell$. By \Autoref{eq:cyclic-count}, the number of indices in
$I_\ell$ for which this output is correct is at most $1+|I_\ell|/k$. 
Summing over all terminal leaves, the total number of indices on which the
algorithm can correctly identify the bump number without encountering a
positive $P$-answer is at most
\begin{align*}
    \sum_\ell
    \left(
        \frac{|I_\ell|}{k}+1
    \right)
    \leq
    \frac{L_k}{k}+2^t.
\end{align*}
Hence the corresponding success probability is at most
\begin{align}\label{eq:terminal-success}
    \frac{1}{k}+\frac{2^t}{L_k}.
\end{align}

Combining \Autoref{eq:positive-P, eq:terminal-success} yields
\begin{align*}
    \P\bigl(\widehat\rho=\rho_k(N)\bigr)
    \leq
    \frac{t}{k}
    +\frac{2^t-1}{L_k}
    +\frac{1}{k}
    +\frac{2^t}{L_k}
    =
    \frac{t+1}{k}
    +
    \frac{2^{t+1}-1}{L_k}.
\end{align*}
which completes the proof.

\end{proofof}

\section{Improper Learning Does not Help}\label{sec:app_proper}
This work focus on proper learning, where the output hypothesis $\hat{f}_n$ is chosen from the hypothesis class $\mathcal{C}$. In contrast, improper learning~\citep{pitt1988computational} allows the learner to choose an hypothesis outside $\mathcal{C}$. As an extreme example, the learner can construct any estimator in $\{0,1\}^\mathcal{X}$. 

At first glance, appears strictly more powerful, since it enlarges the set of admissible estimators. However, in a realizable setting with $\mathcal{C}$ being a VC class of binary classifiers, improper learning does not fundamentally improve the achievable convergence rate. In particular, any improper learner achieving exponential convergence can be converted into a proper learner with the same rate. 

To see this, let $S$ be the same set constructed in the proof of \Autoref{thm:2d}, and choose $\delta=\epsilon$. With probability at least $1-\delta$, $S$ is an $\epsilon$-net. Let $m \defeq \abs{S}$. Suppose an improper learner outputs an estimator $\tilde{f}_n$ and achieves an exponential convergence rate. Then, $\tilde{f}_n$ must correctly label every point in $S$ using $n=O(\ln(m))$ queries. Indeed, if $\tilde{f}_n$ misclassifies any point in $S$, the definition of an $\epsilon$-net implies that the induced error is at least of order $1/m$, contradicting the assumed exponential rate. Since we are in the realizable setting, the labeling of $S$ produced by $\tilde{f}_n$ must coincide with the restriction of some hypothesis $f \in\mathcal C$. Therefore, one may choose any $\hat{f}_n \in \mathcal{C}$ that agrees with $\tilde{f}_n$ on $S$. Since $S$ is an $\epsilon$-net, it implies that $\hat{f}_n$ also achieves the same exponential rate as $\tilde{f}_n$. Consequently, improper learning does not provide an essential advantage over proper learning in our setting. 


\section{Stochastic Query Algorithms} \label{sec:app_determin}
    One may wonder what happens if the query strategy is allowed to be stochastic. That is, the query at each round $i$ can be random. 
    Let us define $\utd$ with a stochastic version of query strategy. In particular, we allow $A$ to output random queries, i.e., $A(X_1,\dots,X_{i}, f(X_1), \dots, f(X_i))$ is a random variable on $\mathcal{X}$. 
    Correspondingly, for some $\eta \in (0,1)$.
    \begin{align*}
        \utd(m) \defeq \sup_{S \in \mathcal{X}^m} \min\biggl\{t \geq 0: \exists A \in \mathcal{A}_t, s.t. \  \sup_{f \in \mathcal{C}} \P( 
        \abs{\{\vs_{U,f}(\mathcal{C}) \}_S}
        >  1) < \eta\biggr\} .
    \end{align*}

    Allowing randomized query strategies potentially enlarges the power of the learner. In particular, the randomized and deterministic versions of $\utd$  need not coincide.
    
    For example, consider a singleton classifiers on $S=\{x_1,\dots,x_m\}$, $f_i(x_j) = \ind_{i=j}$ for $i,j=1,\dots,m$. A deterministic strategy that identifies the labeling of $s$ for every target must make at least $m-1$ queries in the worst case. Indeed, after observing $m-2$ negative labels, two singleton hypotheses remain indistinguishable.
    In contrast, a randomized strategy may query $m-2$ distinct points chosen uniformly at random.
    For any fixed target $f_i$, the labeling of $S$ is uniquely identified whenever $x_i$ is queried. Hence its failure probability is $2/m$. So, for any fixed $\eta$, such a randomize strategy achieves $\utd(m)=m-2$ for sufficiently large $m$.

    This separation, however, does not establish an asymptotic advantage in the regime of primary interest to this paper. In particular, it remains possible that 
    \begin{align*}
        \utd_{\mathrm{rand}}(m,\eta)=O(\log m)
\quad\Longrightarrow\quad
\utd_{\mathrm{det}}(m)=O(\log m).
    \end{align*}
    Determining whether randomization can fundamentally change exponential learnability remains an interesting open question.

    Allowing randomized strategies would also require a probabilistic analogue of the synthesis halving dimension, in which a constant-factor reduction of the candidate space is guaranteed only with high probability. Such a quantity would depend on both the randomized strategy and a failure parameter, and would therefore lose some of the purely structural character that motivates our deterministic synthesis halving dimension. For these reasons, we focus on deterministic query strategies throughout the main paper and leave a systematic study of randomized membership queries for future work.

\section{Logarithm AIC Implies Exponential Learning}\label{sec:app_utd}

In this section, we show that, for any concept class $\mathcal{F}$ with finite VC dimension, $\utd(m)=O(\ln m)$ implies exponential learning. The proof follows the first step of the proof of \Autoref{thm:2d} and relies on Lemma~\ref{lemma:net}. Specifically, Lemma~\ref{lemma:net} ensures that, with high probability, a sample $S$ of size $m=O(\epsilon^{-1}\ln(1/\epsilon))$ forms an $\epsilon$-net. Since $\utd(m)=O(\ln m)$, the labeling of $S$ can be exactly identified using $O(\ln m)=O(\ln(1/\epsilon))$ membership queries. Exact identification on the $\epsilon$-net then guarantees population error at most $\epsilon$. Hence, the error decreases exponentially with the number of membership queries. 

\section{Disjoint Union of Initially-Hard-to-Halve Function Classes}\label{sec:app_union}
    Let $\mathcal{C}_n$ denote the finite concept class proposed in \Autoref{subsec:halve}, which consists of $n$ independent copies of binary trees and is hard to halve initially. 
    We can derive an infinity class $\mathcal{C} = \bigcup_{n=1}^{\infty} \mathcal{C}_n$ by taking a disjoint union of them. Nevertheless, $\atd(\mathcal{C}, m) = O(m)$. To see it, we can choose an arbitrary set $\{n_{t_1}, \dots, n_{t_m}\}$, and form an unlabeled set $S$ by selecting a single query point $q^{i}_{j}$ from each $n_{t_k}, k=1,\dots,m$. 
    As these query points are selected from $m$ disjoint query spaces, their labels are non-informative to each other. As a result, the sample complexity is $m$. 

\end{document}